\documentclass{article}

\usepackage{times}
\usepackage{graphicx} 
\usepackage{subfigure} 
\usepackage{amsfonts}       
\usepackage{amsmath}  
\usepackage{amsthm}
\usepackage{mathtools}
\usepackage{natbib}

\newtheorem{lemma}{Lemma}
\newtheorem{corollary}{Corollary}
\newtheorem{theorem}{Theorem}
\usepackage{algorithm}
\usepackage{algorithmic}
\usepackage{epstopdf}

\usepackage{hyperref}

\newcommand\given[1][]{\:#1\vert\:}

\usepackage[accepted]{icml2017} 

\icmltitlerunning{Consistent On-Line Off-Policy Evaluation}

\begin{document} 

\twocolumn[
\icmltitle{Consistent On-Line Off-Policy Evaluation}

\begin{icmlauthorlist}
	\icmlauthor{Assaf Hallak}{tech}
	\icmlauthor{Shie Mannor}{tech}
\end{icmlauthorlist}

\icmlcorrespondingauthor{Assaf Hallak}{ifogph@gmail.com}
\icmlcorrespondingauthor{Shie Mannor}{shie@ee.technion.ac.il}

\icmlaffiliation{tech}{The Technion, Haifa, Israel}


\vskip 0.3in
]

\printAffiliationsAndNotice{}

\begin{abstract} 
The problem of on-line off-policy evaluation (OPE) has been actively studied in the last decade due to its importance both as a stand-alone problem and as a module in a policy improvement scheme. However, most Temporal Difference (TD) based solutions ignore the discrepancy between the stationary distribution of the behavior and target policies and its effect on the convergence limit when function approximation is applied. In this paper we propose the Consistent Off-Policy Temporal Difference (COP-TD($\lambda$, $\beta$)) algorithm that addresses this issue and reduces this bias at some computational expense. We show that COP-TD($\lambda$, $\beta$) can be designed to converge to the same value that would have been obtained by using on-policy TD($\lambda$) with the target policy. Subsequently, the proposed scheme leads to a related and promising heuristic we call log-COP-TD($\lambda$, $\beta$). Both algorithms have favorable empirical results to the current state of the art on-line OPE algorithms. Finally, our formulation sheds some new light on the recently proposed Emphatic TD learning.
\end{abstract} 

\section{Introduction}

Reinforcement Learning (RL) techniques were successfully applied in fields such as robotics, games, marketing and more \citep{kober2013reinforcement, al2015application, barrett2013applying}. We consider the problem of off-policy evaluation (OPE) -- assessing the performance of a complex strategy without applying it. An OPE formulation is often considered in domains with limited sampling capability. For example, marketing and recommender systems \citep{theocharous2013lifetime, theocharous2015personalized} directly relate policies to revenue. A more extreme example is drug administration, as there are only few patients in the testing population, and sub-optimal policies can have life threatening effects \citep{hochberg2016encouraging}. OPE can also be useful as a module for policy optimization in a policy improvement scheme \citep{thomas2015high}.


In this paper, we consider the OPE problem in an on-line setup where each new sample is immediately used to update our current value estimate of some previously unseen policy. We propose and analyze a new algorithm called COP-TD($\lambda$,$\beta$) for estimating the value of the target policy; COP-TD($\lambda$,$\beta$) has the following properties:
\begin{enumerate}
	\item Easy to understand and implement on-line.
	\item Allows closing the gap to consistency such that the limit point is the same that would have been obtained by on-policy learning with the target policy.
	\item Empirically comparable to state-of-the art algorithms.	
\end{enumerate}
Our algorithm resembles \cite{SuttonMW15}'s Emphatic TD that was extended by \cite{hallak2015generalized} to the general parametric form ETD($\lambda$,$\beta$). We clarify the connection between the algorithms and compare them empirically. Finally, we introduce an additional related heuristic called Log-COP-TD($\lambda$,$\beta$) and motivate it. 

\section{Notations and Background}
We consider the standard discounted Markov Decision Process (MDP) formulation \cite{BT96} with a single long trajectory. Let $M = (\mathcal{S}, \mathcal{A}, \mathcal{P}, \mathcal{R}, \zeta, \gamma)$ be an MDP where $\mathcal{S}$ is the finite state space and $\mathcal{A}$ is the finite action space. The parameter $\mathcal{P}$ sets the transition probabilities $\Pr(s'|s, a)$ given the previous state $s\in \mathcal{S}$ and action $a\in \mathcal{A}$, where the first state is determined by the distribution $\zeta$. The parameter $\mathcal{R}$ sets the reward distribution $r(s,a)$ obtained by taking action $a$ in state $s$ and $\gamma$ is the discount factor specifying the exponential reduction in reward with time. 

The process advances as follows: A state $s_0$ is sampled according to the distribution $\zeta(s)$. Then, at each time step $t$ starting from $t=0$ the agent draws an action $a_t$ according to the stochastic behavior policy $\mu(a|s_t)$, a reward $r_t \doteq r(s_t, a_t)$ is accumulated by the agent, and the next state $s_{t+1}$ is sampled using the transition probability $\Pr(s'|s_t, a_t)$. 

The expected discounted accumulated reward starting from a specific state and choosing an action by some policy $\pi$ is called the value function, which is also known to satisfy the Bellman equation in a vector form:
\begin{equation*}
	V^\pi(s) = \mathbb{E}_\pi \left[ \sum_{t=0}^\infty \gamma^t r_t \given[\Big]  s_0 = s \right], \quad T_\pi V \doteq R_\pi + \gamma P_\pi V,
\end{equation*}
where $\left[ R_\pi \right]_{s} \doteq \mathbb{E}_\pi \left[ r(s,\pi(s))  \right]$ and $\left[ P_\pi \right]_{s,s'} \doteq \mathbb{E}_\pi \left[ \Pr(s'|s,\pi(s))  \right]$ are the policy induced reward vector and transition probability matrix respectively; $T_\pi$ is called the Bellman operator. The problem of estimating $V^\pi(s)$ from samples is called policy evaluation. If the target policy $\pi$ is different than the behavior policy $\mu$ which generated the samples, the problem is called off-policy evaluation (OPE). The TD($\lambda$) \citep{sutton1988learning} algorithm is a standard solution to on-line on-policy evaluation: Each time step the temporal difference error updates the current value function estimate, such that eventually the stochastic approximation process will converge to the true value function. The standard form of TD($\lambda$) is given by:
\begin{equation}\label{Eq:TD}
	\begin{split}
		R_{t, s_t}^{(n)} =& \sum_{i=0}^{n-1} \gamma^i r_{t+i} + \gamma^n \hat{V}_t(s_{t+n}), \\
		R_{t, s_t}^\lambda =& (1-\lambda)\sum_{n=0}^\infty \lambda^n R_{s_t}^{(n+1)}, \\
		\hat{V}_{t+1}(s_t) =& \hat{V}_t(s_t) + \alpha_t \left(  R_{t, s_t}^\lambda - \hat{V}_t(s_t) \right),
	\end{split}
\end{equation}
where $\alpha_t$ is the step size. The value $R_{t, s_t}^{(n)}$ is an estimate of the current state's $V(s_t)$, looking forward $n$ steps, and $R_{t, s_t}^\lambda$ is an exponentially weighted average of all of these estimates going forward till infinity. Notice that Equation \ref{Eq:TD} does not specify an on-line implementation since $R_{t, s_t}^{(n)}$ depends on future observations, however there exists a compact on-line implementation using eligibility traces (\citet{BT96} for on-line TD($\lambda$), and \citet{sutton2014new}, \citet{SuttonMW15} for off-policy TD($\lambda$)). The underlying operator of TD($\lambda$) is given by:
\begin{equation*}
	\begin{split}
	T_\pi^\lambda V &= (1-\lambda) \sum_{n=0}^\infty \lambda^n \left( \sum_{i=0}^n \gamma^i P_\pi^i  R_\pi +  \gamma^{n+1} P_\pi^{n+1}  V \right) \\
	&= (1-\lambda) (I-\lambda T_\pi)^{-1} T_\pi V,
	\end{split}
\end{equation*}
and is a $\frac{\gamma (1-\lambda)}{1-\lambda\gamma}$-contraction \cite{Ber2012DynamicProgramming}.


We denote by $d_\mu(s)$ the stationary distribution over states induced by taking the policy $\mu$ and mark $D_\mu = diag(d_\mu)$. Since we are concerned with the behavior at infinite horizon, we assume $\zeta(s)=d_\mu(s)$. In addition, we assume that the MDP is ergodic for the two specified policies $\mu,\pi$ so $\forall s\in \mathcal{S}: d_\mu(s) >0, d_\pi(s) > 0 $ and that the OPE problem is proper -- $\pi(a|s) > 0 \Rightarrow \mu(a|s) > 0 $.

When the state space is too large to hold $V ^\pi(s)$, a linear function approximation scheme is used: $V^\pi(s) \approx \theta_\pi^\top \phi(s)$, where $\theta$ is the optimized weight vector and $\phi(s)$ is the feature vector of state $s$ composed of $k$ features. We denote by $\Phi \in \mathbb{R}^{S, k}$ the matrix whose lines consist of the feature vectors for each state and assume its columns are linearly independent. 

TD($\lambda$) can be adjusted to find the fixed point of $\Pi_{d_\pi}T_\pi^\lambda$ where $\Pi_{d_\pi}$ is the projection to the subspace spanned by the features with respect to the $d_\pi$-weighted norm \cite{sutton_reinforcement_1998}:

\begin{equation*}
	\begin{split}
		R_{t, s_t}^{(n)} =& \sum_{i=0}^{n-1} \gamma^i r_{t+i} + \gamma^n \theta_t^\top \phi(s_{t+n}), \\
		R_{t, s_t}^\lambda =& (1-\lambda)\sum_{n=0}^\infty \lambda^n R_{s_t}^{(n+1)}, \\
		\theta_{t+1} =& \theta_t + \alpha_t \left(  R_{t, s_t}^\lambda - \theta^\top_t \phi(s_t) \right)\phi(s_t).
	\end{split}
\end{equation*}

Finally, we define OPE-related quantities:
\begin{equation}\label{Eq:Gamma}
	\begin{split}
		\rho_t \doteq \frac{\pi (a_t|s_t)}{\mu (a_t|s_t)}, \quad \Gamma_t^n \doteq& \prod_{i=0}^{n-1}\rho_{t-1-i}, \quad \rho_d (s) \doteq \frac{d_\pi(s)}{d_\mu(s)},
	\end{split}
\end{equation}
we call $\rho_d$ the covariate shift ratio (as denoted under different settings by \cite{hachiya2012importance}). 

We summarize the assumptions used in the proofs:
\begin{enumerate}
	\item Under both policies the induced Markov chain is ergodic.
	\item The first state $s_0$ is distributed according to the stationary distribution of the behavior policy $d_\mu(s)$.
	\item The problem is proper: $\pi(a|s) > 0 \Rightarrow \mu(a|s) > 0 $.
	\item The feature matrix $\Phi$ has full rank $k$.
\end{enumerate}
Assumption 1 is commonly used for convergence theorems as it verifies the value function is well defined on all states regardless of the initial sampled state. Assumption 2 can be relaxed since we are concerned with the long-term properties of the algorithm past its mixing time -- we require it for clarity of the proofs. Assumption 3 is required so the importance sampling ratios will be well defined. Assumption 4 guarantees the optimal $\theta$ is unique which greatly simplifies the proofs.

\section{Previous Work}
We can roughly categorize previous OPE algorithms to two main families. Gradient based methods that perform stochastic gradient descent on error terms they want to minimize. These include GTD \citep{sutton2009convergent}, GTD-2, TDC \citep{sutton2009fast} and HTD \citep{white2016investigating}. The main disadvantages of gradient based methods are (A) they usually update an additional error correcting term, which means another time-step parameter needs to be controlled; and (B) they rely on estimating non-trivial terms, an estimate that tends to converge slowly. The other family uses importance sampling (IS) methods that correct the gains between on-policy and off-policy updates using the IS-ratios $\rho_t$'s. Among these are full IS \citep{precup2001off} and ETD($\lambda$,$\beta$) \citep{SuttonMW15}. These methods are characterized by the bias-variance trade-off they resort to -- navigating between biased convergent values (or even divergent), and very slow convergence stemming from the high variance of IS correcting factors (the $\rho_t$ products). There are also a few algorithms that fall between the two, for example TO-GTD \citep{van2014off} and WIS-TD($\lambda$) \citep{mahmood2015off}.

A comparison of these algorithms in terms of convergence rate, synergy with function approximation and more is available in \citep{white2016investigating, geist2014off}. We focus in this paper on the limit point of the convergence. For most of the aforementioned algorithms, the process was shown to converge almost surely to the fixed point of the projected Bellman operator $\Pi_d T_\pi$ where $d$ is some stationary distribution (usually $d_\mu$), however the $d$ in question was never\footnote{Except full IS, however its variance is too high to be applicable in practice.} $d_\pi$ as we would have obtained from running on-policy TD with the target policy. The algorithm achieving the closest result is ETD($\lambda$,$\beta$) which replaced $d$ with $f=\left(I - \beta P^\top_\pi \right)^{-1} d_\mu$, where $\beta$ trades-off some of the process' variance with the bias in the limit point. 
Hence, our main contribution is a consistent algorithm which can converge to the same value that would have been obtained by running an on-policy scheme with the same policy. 

\section{Motivation}\label{Sec:Motiv}
Here we provide a motivating example showing that even in simple cases with ``close'' behavior and target policies, the two induced stationary distributions can differ greatly. Choosing a specific linear parameterization further emphasizes the difference between applying on-policy TD with the target policy, and applying inconsistent off-policy TD. 

Assume a chain MDP with numbered states $1,2,..|S|$, where from each state $s$ you can either move left to state $s-1$, or right to state $s+1$. If you've reached the beginning or the end of the chain (states $1$ or $|S|$) then taking a step further does not affect your location. Assume the behavior policy moves left with probability $0.5+\epsilon$, while the target policy moves right with probability $0.5+\epsilon$. It is easy to see that the stationary distributions are given by:
\begin{equation*}
	d_\mu (s) \propto \left( \frac{0.5-\epsilon}{0.5+\epsilon} \right)^s, \quad\quad d_\pi (s) \propto \left( \frac{0.5+\epsilon}{0.5-\epsilon} \right)^s.
\end{equation*}
For instance, if we have a length $100$ chain with $\epsilon=0.01$, for the rightmost state we have $d_\mu(|S|) \approx 8 \cdot 10^{-4}, d_\pi (|S|) \approx 0.04$. Let's set the reward to be $1$ for the right half of the chain, so the target policy is better since it spends more time in the right half. The value of the target policy in the edges of the chain for $\gamma=0.99$ is $V^{\pi}(1)= 0.21, V^{\pi}(100)=99.97$. 

Now what happens if we try to approximate the value function using one constant feature $\phi(s) \equiv 1$? The fixed point of $\Pi_{d_\mu}T_\pi$ is $\theta=11.92$, while the fixed point of $\Pi_{d_\pi}T_\pi$ is $\theta=88.08$ -- a substantial difference. The reason for this difference lies in the emphasis each projection puts on the states: according to $\Pi_{d_\mu}$, the important states are in the left half of the chain -- these with low value function, and therefore the value estimation of all states is low. However, according to $\Pi_{d_\pi}$ the important states are concentrated on the right part of the chain since the target policy will visit these more often. Hence, the estimation error is emphasized on the right part of the chain and the value estimation is higher. When we wish to estimate the value of the target policy, we want to know what will happen if we deploy it instead of the behavior policy, thus taking the fixed point of $\Pi_{d_\pi}T_\pi$ better represents the off-policy evaluation solution. 

\section{COP-TD($\lambda$, $\beta$)} \label{Sec:rTD}
Most off-policy algorithms multiply the TD summand of TD($\lambda$) with some value that depends on the history and the current state. For example, full IS-TD by \cite{precup2001off} examines the ratio between the probabilities of the trajectory under both policies:
\begin{equation}\label{Eq:fullIS}
	\frac{P_\pi(s_0,a_0,s_1,\dots, s_t, a_t)}{P_\mu(s_0,a_0,s_1,\dots, s_t, a_t)} = \prod_{m=0}^t \rho_m \ = \Gamma^t_t \rho_t.
\end{equation}
In problems with a long horizon, or these that start from the stationary distribution, we suggest using the time-invariant covariate shift $\rho_d$ multiplied by the current $\rho_t$. The intuition is the following: We would prefer using the probabilities ratio given in Equation \ref{Eq:fullIS}, but it has very high variance, and after many time steps we might as well look at the stationary distribution ratio instead. This direction leads us to the following update equations:
\begin{equation} \label{Eq:Target}
	\begin{split}
	\theta_{t+1} &= \\
	 & \theta_t + \alpha_t \rho_d(s_t) \rho_t  \left( r_t + \theta^\top_t (\gamma\phi(s_{t+1})  -  \phi(s_t) ) \right)\phi(s_t).
	\end{split}
\end{equation}

\begin{lemma} \label{Lem:Main}
	If the $\alpha_t$ satisfy $\sum_{t=0}^\infty \alpha_t = \infty, \sum_{t=0}^\infty \alpha^2_t < \infty$ then the process described by Eq. (\ref{Eq:Target}) converges almost surely to the fixed point of $\Pi_\pi T_\pi V = V$. 
\end{lemma}
The proof follows the ODE method \cite{kushner2003stochastic} similarly to \citet{tsitsiklis1997analysis} (see the appendix for more details).

Since $\rho_d(s)$ is generally unknown, it is estimated using an additional stochastic approximation process. In order to do so, we note the following Lemma:

\begin{lemma} \label{Lem:Gamma}
	Let $\widehat{\rho_d}$ be an unbiased estimate of $\rho _d$, and for every $n=0,1,\dots,t$ define  $\tilde{\Gamma}_t^n\doteq \widehat{\rho_d} (s_{t-n}) \Gamma^n_t$. Then:
	\begin{equation*}
	\mathbb{E}_\mu \left[ \tilde{\Gamma}^n_t | s_t \right] = 	\rho_d(s_t).
	\end{equation*}
	\end{lemma}
%

For any state $s_t$ there are $t\rightarrow\infty$ such quantities $\{ \tilde{\Gamma}_t^n \}_{n=0}^t$, where we propose to weight them similarly to TD($\lambda$):
\begin{equation} \label{Eq:GammaBetaGal}
	\tilde{\Gamma}_t^\beta = (1-\beta) \sum_{n=0}^\infty \beta^n \tilde{\Gamma}^{n+1}_t.
\end{equation}
Note that $\rho_d(s)$, unlike $V(s)$, is restricted to a close set since its $d_\mu$-weighted linear combination is equal to $1$ and all of its entries are non-negative; We denote this $d_\mu$-weighted simplex by $\Delta_{d_\mu}$, and let $\Pi_{\Delta_{d_\mu}}$ be the (non-linear) projection to this set with respect to the Euclidean norm ($\Pi_{\Delta_{d_\mu}}$ can be calculated efficiently, \cite{chen2011projection}). Now, we can devise a TD algorithm which estimates $\rho_d$ and uses it to find $\theta$, which we call COP-TD($0$, $\beta$) (Consistent Off-Policy TD).

\begin{algorithm}                      
	\caption{COP-TD($0$,$\beta$), Input: $\theta_0$, $ \widehat{ \rho_d }_{,0},\quad\quad\quad$}          
	\label{alg:rTD}                        
	\begin{algorithmic}[1]                 
		\STATE Init: $F_0 = 0, \quad n^\beta_0 = 1, \quad N(s) =0$
		\FOR{$t=1,2,...$}
		\STATE Observe $s_t, a_t, r_t, s_{t+1}$
		\STATE \textcolor{blue}{Update normalization terms:}
		\STATE $N(s_t) =  N(s_t) + 1 ,  \quad \forall s\in\mathcal{S}: \hat{d}_\mu(s) = \frac{N(s)}{t}$
		\STATE $n^\beta_t = \beta n^\beta_t + 1$
		\STATE \textcolor{blue}{Update $\Gamma_t^n$'s weighted average:	}   
		\STATE $F_t = \rho_{t-1} (\beta F_{t-1} + e_{s_{t-1}})$
		\STATE \textcolor{blue}{Update \& project by $\rho_d$'s TD error:}
		\STATE $\delta^d_t =  \underbrace{ \frac{F^\top_t \widehat{ \rho _d }_{,t}}{n^\beta_t} }_{\rightarrow \tilde{\Gamma}_t^\beta} - \widehat{ \rho _d }_{,t}(s_t)$ 
		\STATE $\widehat{ \rho_d }_{,t+1} = \Pi_{\Delta_{\hat{d}_\mu}} \left( \widehat{ \rho _d }_{,t} + \alpha^d_t \delta^d_t e_{s_t} \right)$ 
		\STATE \textcolor{blue}{Off-policy TD($0$):}
		\STATE $\delta_t = r_t + \theta^\top_t (\gamma\phi(s_{t+1})  -  \phi(s_t) ) $
		\STATE $\theta_{t+1} = \theta_t + \alpha_t \widehat{\rho_d}_{,t+1}(s_t) \rho_t \delta_t \phi(s_t)$
		\ENDFOR
	\end{algorithmic}
\end{algorithm}

Similarly to the Bellman operator for TD-learning, we define the underlying COP-operator $Y$ and its $\beta$ extension:
\begin{equation}
	\begin{split}
		Yu &= D^{-1}_\mu P^\top_\pi D_\mu u, \\ 
		Y^\beta u &= (1-\beta) D^{-1}_\mu P^\top_\pi (I - \beta P^\top_\pi)^{-1} D_\mu u.
	\end{split}
\end{equation}
The following Lemma may give some intuition on the convergence of the $\rho_d$ estimation process:
\begin{lemma} \label{Lem:Contraction}
	Under the ergodicity assumption, denote the eigenvalues of $P_\pi$ by $0 \leq \dots \leq | \xi_2 | < \xi_1 = 1$. Then $Y^\beta$ is a $\max_{i \neq 1} \frac{(1-\beta)|\xi_i|}{|1-\beta\xi_i|}<1$-contraction in the $L_2$-norm on the orthogonal subspace to $\rho_d$, and $\rho_d$ is a fixed point of $Y^\beta$.
\end{lemma}
The technical proof is given in the appendix.

\begin{theorem}\label{Thm:COP1}
	If the step sizes satisfy $\sum_t \alpha_t = \sum_t \alpha^d_t = \infty, \sum_t (\alpha^2_t + (\alpha_t^d)^2) < \infty, \frac{\alpha_t}{\alpha^d_t} \rightarrow 0, t \alpha^d_t \rightarrow 0$, and $\mathbb{E} \left[ (\beta^n \Gamma_t^n)^2 | s_t \right] \leq C$ for some constant $C$ and every $t$ and $n$, then after applying COP-TD($0$, $\beta$), $\widehat{ \rho_d }_{,t}$  converges to $\rho_d$ almost surely, and $\theta_t$ converges to the fixed point of $\Pi_\pi T_\pi V$. 
\end{theorem}

Notice that COP-TD($0$, $\beta$) given in Alg. \ref{alg:rTD} is infeasible in problems with large state spaces since $\rho_d \in \mathbb{R}^{|\mathcal{S}|}$. Like TD($\lambda$), we can introduce linear function approximation: represent $\rho_d(s) \approx \theta_\rho ^\top \phi_\rho(s)$ where $\theta_\rho$ is a weight vector and $\phi_\rho(s)$ is the off-policy feature vector and adjust the algorithm accordingly. For $\widehat{ \rho_d }$ to still be contained in the set $\Delta_{d_\mu}$, we pose the requirement on the feature vectors: $\phi_\rho(s) \in \mathbb{R}_+^k$, and $\sum_s d_\mu(s) \theta_\rho^\top \phi_\rho(s) = 1$ $\big($noted as the simplex projection $\Pi_{\Delta_{\mathbb{E}_\mu [\phi_\rho(s)] }}\big)$. In practice, the latter requirement can be approximated: $\sum_s d_\mu(s) \theta_\rho^\top \phi_\rho(s) \approx \frac{1}{t} \theta_\rho^\top \sum_t \phi_\rho(s_t) = 1$ resulting in an extension of the previously applied $d_\mu$ estimation (step 5 in COP-TD($0$, $\beta$)). We provide the full details in Algorithm \ref{alg:rTD_FA}, which also incorporates non-zero $\lambda$ $\big($similarly to ETD($\lambda$,$\beta$)$\big)$.

\begin{algorithm}             
	\caption{COP-TD($\lambda$,$\beta$) with Function Approximation, Input: $\theta_0$, $\theta_{\rho, 0}$}          
	\label{alg:rTD_FA}   
	\begin{algorithmic}[1]      
		\STATE Init: $F_0 = \underline{0}, \quad n^\beta_0 = 1, \quad N_\phi =\underline{0}, \quad e_0 = \underline{0}$
		\FOR{$t=1,2,...$}
		\STATE Observe $s_t, a_t, r_t, s_{t+1}$
		\STATE \textcolor{blue}{Update normalization terms:}
		\STATE $n^\beta_t = \beta n^\beta_t + 1, \quad N_\phi =  N_\phi + \phi_\rho(s_t), \quad \hat{d}_{\phi_\rho} = \frac{N_\phi}{t}$
		\STATE \textcolor{blue}{Update $\Gamma_t^n$'s weighted average:	}   	    
		\STATE $F_t = \rho_{t-1} (\beta F_{t-1} + \phi_\rho(s_{t-1}))$	    
		\STATE \textcolor{blue}{Update \& project by $\rho_d$'s TD error:}
		\STATE $\delta^d_t =  \theta_{\rho ,t-1}^\top \left( \frac{F_t} {n^\beta_t} -  \phi_\rho(s_t) \right)$
		\STATE $\theta_{\rho ,t+1} = \Pi_{\Delta_{\hat{d}_{\phi_\rho}}} \left( \theta_{\rho ,t} + \alpha^d_t \delta^d_t \phi_\rho(s_t) \right)$
		\STATE \textcolor{blue}{Off-policy TD($\lambda$):}
		\STATE $M_t = \lambda + (1-\lambda) \theta_{\rho ,t+1}^\top \phi_\rho(s_t) $
		\STATE $e_t = \rho_t \left( \lambda \gamma e_t + M_t \phi(s_{t+1})  \right)$
		\STATE $\delta_t = r_t + \theta^\top_t (\gamma\phi(s_{t+1})  -  \phi(s_t) )$
		\STATE $\theta_{t+1} = \theta_t + \alpha_t \delta_t e_t$
		\ENDFOR
	\end{algorithmic}
\end{algorithm}

\begin{theorem}\label{Thm:COP2}
	If the step sizes satisfy $\sum_t \alpha_t = \sum_t \alpha^d_t = \infty, \sum_t (\alpha^2_t + (\alpha_t^d)^2) < \infty, \frac{\alpha_t}{\alpha^d_t} \rightarrow 0, t \alpha^d_t \rightarrow 0$, and $\mathbb{E} \left[ (\beta^n \Gamma_t^n)^2 | s_t \right] \leq C$ for some constant $C$ and every $t,n$, then after applying COP-TD($0$, $\beta$) with function approximation satisfying $\phi_\rho(s) \in \mathbb{R}_+^k$, $\widehat{ \rho_d }_{,t}$ converges to the fixed point of $\Pi_{\Delta_{\mathbb{E}_\mu [\phi_\rho] }} \Pi_{\phi_\rho} Y^\beta $ denoted by $\rho^\text{COP}_d$ almost surely, and if $\theta_t$ converges it is to the fixed point of $\Pi_{d_\mu \circ \rho^\text{COP}_d} T_\pi V$, where $\circ$ is a coordinate-wise product of vectors.
\end{theorem}
The proof is given in the appendix and also follows the ODE method. Notice that a theorem is only given for $\lambda=0$, convergence results for general $\lambda$ should follow the work by \citet{yu2015etd}. 

A possible criticism on COP-TD($0$,$\beta$) is that it is not actually consistent, since in order to be consistent the original state space has to be small, in which case every off-policy algorithm is consistent as well. Still, the dependence on another set of features allows to trade-off accuracy with computational power in estimating $\rho_d$ and subsequently $V$. Moreover, smart feature selection may further reduce this gap, and COP-TD($0$, $\beta$) is still the first algorithm addressing this issue. We conclude with linking the error in $\rho_d$'s estimate with the difference in the resulting $\theta$, which suggests that a well estimated $\rho_d$ results in consistency:


\begin{corollary}
	Let $0<\epsilon<1$. If $ (1-\epsilon) \rho_d \leq \rho_d^\text{COP} \leq (1+\epsilon) \rho_d$, then the fixed point of COP-TD($0$,$\beta$)  with function approximation $\theta^\text{COP}$ satisfies the following, where $\| \cdot \|_\infty$ is the $L_\infty$ induced norm:
	\begin{equation}
		\begin{split}
			\| \theta^* - \theta^\text{COP} \|_\infty & \leq \\
			\epsilon \| A_\pi^{-1} \Phi^\top & \| _\infty \left(   R_\text{max}  +  (1+\gamma) \| \Phi \|_\infty \| \theta^\text{COP} \|_\infty \right),
		\end{split}
	\end{equation}
	where $A_\pi = \Phi^\top D_\pi (I - \gamma P_\pi) \Phi$, and $\theta^*$ sets the fixed point of the operator $\Pi_{d_\pi} T_\pi V$.
\end{corollary}

\subsection{Relation to ETD($\lambda$, $\beta$)} \label{Subsec:ETD}

Recently, \citet{SuttonMW15} had suggested an algorithm for off-policy evaluation called Emphatic TD. Their algorithm was later on extended by \citet{hallak2015generalized} and renamed ETD($\lambda$, $\beta$), which was shown to perform extremely well empirically by \citet{white2016investigating}. ETD($0$, $\beta$) can be represented as:
\begin{equation}\label{Eq:ETD}
	\begin{split}
		F_t &= (1-\beta) \sum_{n=0}^\infty \beta^n \Gamma_t^n, \\ 
		\hat{V}_{t+1}(s_t) &= \hat{V}_t(s_t) + \alpha_t F_t \rho_t  \left( r_t + \theta^\top_t (\gamma\phi(s_{t+1})  -  \phi(s_t) ) \right).
	\end{split}
\end{equation}

As mentioned before, ETD($\lambda$, $\beta$) converges to the fixed point of $\Pi_f T_\pi^\lambda$ \citep{yu2015etd}, where $f = \mathbb{E} \left[ F_t | s_t \right] = (I - \beta P_\pi)^{-1} d_\mu$. Error bounds can be achieved by showing that the operator $\Pi_f T_\pi^\lambda$ is a contraction under certain requirements on $\beta$ and that the variance of $F_t$ is directly related to $\beta$ as well \citep{hallak2015generalized} (and thus affects the convergence rate of the process). 

When comparing ETD($\lambda$,$\beta$)'s form to COP-TD($\lambda$,$\beta$)'s, instead of spending memory and time resources on a state/feature-dependent $F_t$, ETD($\lambda$,$\beta$) uses a one-variable approximation. The resulting $F_t$ is in fact a one-step estimate of $\rho_d$, starting from $\widehat{\rho_d}(s) \equiv 1$ (see Equations \ref{Eq:GammaGal}, \ref{Eq:ETD}), up to a minor difference: $F_t^\text{ETD} = \beta F_t^\text{COP-TD} + 1$ (which following our logic adds bias to the estimate \footnote{We have conducted several experiments with an altered ETD and indeed obtained better results compared with the original, these experiments are outside the scope of the paper.}).

Unlike ETD($\lambda$, $\beta$), COP-TD($\lambda$,$\beta$)'s effectiveness depends on the available resources. The number of features $\phi_\rho(s)$ can be adjusted accordingly to provide the most affordable approximation. The added cost is fine-tuning another step-size, though $\beta$'s effect is less prominent.
%


\section{The Logarithm Approach for Handling Long Products}
We now present a heuristic algorithm which works similarly to COP-TD($\lambda$, $\beta$). Before presenting the algorithm, we explain the motivation behind it.

\subsection{Statistical Interpretation of TD($\lambda$)} \label{Subsec:stat}
\citet{NIPS2011_4472} suggested a statistical interpretation of TD($\lambda$). They show that under several assumptions the TD($\lambda$) estimate $R^\lambda_{s_t}$ is the maximum likelihood estimator of $V(s_t)$ given $R_{s_t}^n$: (1) Each $R_{s_t}^n$ is an unbiased estimator of $V(s_t)$; (2) The random variables $R_{s_t}^n$ are independent and specifically uncorrelated; (3) The random variables $R_{s_t}^n$ are jointly normally distributed; and (4) The variance of each $R_{s_t}^n$ is proportional to $\lambda^n$.


Under Assumptions 1-3 the maximum likelihood estimator of $V(s)$ given its previous estimate can be represented as a linear convex combination of $R_{s_t}^n$ with weights:
\begin{equation}
	w_n= \frac{\left[ \mathrm{Var} \left(R_{s_t}^{(n)}\right) \right]^{-1}}{\sum_{m=0}^\infty \left[ \mathrm{Var}\left(R_{s_t}^{(m)}\right) \right]^{-1}}.
\end{equation}


Subsequently, in \citet{NIPS2011_4472} Assumption $4$ was relaxed and instead a closed form approximation of the variance was proposed. In a follow-up paper by \citet{NIPS2015_5807}, the second assumption was also removed and the weights were instead given as: $w_n = \frac{\textbf{1}^\top cov(\textbf{R}_{s_t}) \textbf{e}_n}{\textbf{1}^\top cov(\textbf{R}_{s_t}) \textbf{1}}$, where the covariance matrix can be estimated from the data, or otherwise learned through some parametric form. 

While both the approximated variance and learned covariance matrix solutions improve performance on several benchmarks, the first uses a rather crude approximation, and the second solution is both state-dependent and based on noisy estimates of the covariance matrix. In addition, there aren't efficient on-line implementations since all past weights should be recalculated to match a new sample. Still, the suggested statistical justification is a valuable tool in assessing the similar role of $\beta$ in ETD($\lambda$, $\beta$).

\subsection{Variance Weighted $\Gamma_t^n$}
As was shown by \citet{NIPS2011_4472}, we can use state-dependent weights instead of $\beta$ exponents to obtain better estimates. The second moments are given explicitly as follows\footnote{The covariances can be expressed analytically as well, for clarity we drop this immediate result.}: $\mathbb{E} \left[  \left( \Gamma_t^n \right)^2 |s_t \right] = \frac{d_\mu^\top \tilde{P}^{n-1} \textbf{e}_{s_t}}{d_\mu (s_t)}$, where $\left[ \tilde{P} \right]_{s,s'} = \sum_{a \in \mathcal{A}} \frac{\pi^2 (a|s)}{\mu (a|s)}P(s'|s,a)$. 

These can be estimated for each state separately. Notice that the variances increase exponentially depending on the largest eigenvalue of $\tilde{P}$ (as Assumption 4 dictates), but this is merely an asymptotic behavior and may be relevant only when the weights are already negligible. Hence, implementing this solution on-line should not be a problem with the varying weights, as generally only the first few of these are non-zero. While this solution is impractical in problems with large state spaces parameterizing or approximating these variances (similarly to \citet{NIPS2015_5807}) could improve performance in specific applications. 


\subsection{Log-COP-TD($\lambda$, $\beta$)} \label{Subsec:logrTD}
Assumption 3 in the previous section is that the sampled estimators ($R^{(n)},\Gamma_t^n$) are normally distributed. For on policy TD($\lambda$), this assumption might seem not too harsh as the estimators $R^{(n)}$ represent growing \textbf{sums} of random variables. However, in our case the estimators $\Gamma_t^n$ are growing \textbf{products} of random variables. To correct this issue we can define new estimators using a logarithm on each $\tilde{\Gamma}_t^n$:

\begin{equation} \label{Eq:logApprox}
	\begin{split}
		\log \left[ \rho_d(s_t) \right] &= \log \left[ \mathbb{E}\left[  \widehat{ \rho _d } (s_{t-m})  \prod_{k=t-m}^{t-1} \rho_k \bigm| s_t \right] \right] \\
		&\approx  \log\left[  \widehat{ \rho _d } (s_{t-m}) \right] + \sum_{k=t-m}^{t-1} \mathbb{E}\left[ \log \left[ \rho_k \right] | s_t \right].
	\end{split}
\end{equation}

This approximation is crude -- we could add terms reducing the error through Taylor expansion, but these would be complicated to deal with. Hence, we can relate to this method mainly as a well-motivated heuristic. 

Notice that this formulation resembles the standard MDP formulation, only with the corresponding "reward" terms $\log [\rho_t]$ going backward instead of forward, and no discount factor. Unfortunately, without a discount factor we cannot expect the estimated value to converge, so we propose using an artificial one $\gamma_{\log}$. We can incorporate function approximation for this formulation as well. Unlike COP-TD($\lambda$, $\beta$), we can choose the features and weights as we wish with no restriction, besides the linear constraint on the resulting $\rho_d$ through the weight vector $\theta_\rho$. This can be approximately enforced by normalizing $\theta_\rho$ using $\frac{X}{t} \doteq \frac{1}{t}\sum_t \exp ( \theta_{\rho,t}^\top \phi(s_t))$ (which should equal $1$ if we were exactly correct). We call the resulting algorithm Log-COP-TD($\lambda$,$\beta$).

\begin{algorithm}             
	\caption{Log-COP-TD($\lambda$,$\beta$) with Function Approximation, Input: $\theta_0$,$\theta_{\rho,0}$}          
	\label{alg:logrTD_FA}      
	\begin{algorithmic}[1]   
		\STATE Init: $F_0 = 0, \quad n_0(\beta) = 1, \quad N(s) =0$
		\FOR{$t=1,2,...$}
		\STATE Observe $s_t, a_t, r_t, s_{t+1}$
		\STATE \textcolor{blue}{Update normalization terms:}
		\STATE $n^\beta_t = \beta n_t^\beta + 1, \quad N_\phi =  \gamma_{\log}(\beta N_\phi + \phi_\rho(s_t)), \quad X = X + \exp ( \theta_{\rho,t}^\top \phi(s_t))$
		\STATE \textcolor{blue}{Update $\log(\Gamma_t^n)$'s weighted average:	}   	
		\STATE $F_t = \beta \gamma_{\log} F_{t-1} + n_t^\beta \log [\rho(s_{t-1})]$
		\STATE \textcolor{blue}{Update \& project by $\log(\rho_d)$'s TD error:}
		\STATE $\delta^d_t =  \frac{F_t}{n_t^\beta} + \theta^\top_{\rho ,t} \left( \frac{N_\phi}{n_t^\beta} - \phi_\rho(s_t)  \right) $
		\STATE $\theta_{\rho ,t+1} =  \theta_{\rho ,t} + \alpha^d_t \delta^d_t \phi_\rho(s_t)$
		\STATE \textcolor{blue}{Off-policy TD($\lambda$):}
		\STATE $M_t = \lambda + (1-\lambda) \exp\left( \theta^\top_{\rho ,t+1} \phi_\rho(s_t) \right)/ (X/t) $	    
		\STATE $e_t = \rho_t \left( \lambda \gamma e_t + M_t \phi(s_{t+1})  \right)$
		\STATE $\delta_t = r_t + \theta^\top_t (\gamma\phi(s_{t+1})  -  \phi(s_t) )$
		\STATE $\theta_{t+1} = \theta_t + \alpha_t \delta_t e_t$
		\ENDFOR
	\end{algorithmic}
\end{algorithm}

\subsection{Using the Original Features}
An interesting phenomenon occurs when the behavior and target policies employ a feature based Boltzmann distribution for choosing the actions: $\mu(a|s) = \exp \left(\theta_{a, \mu}^\top \phi(s) \right)$, and $\pi(a|s) = \exp \left(\theta_{a, \pi}^\top \phi(s) \right)$, where a constant feature is added to remove the (possibly different) normalizing constant. Thus, $\log(\rho_t) = (\theta_{a, \pi} - \theta_{a, \mu} )^\top \phi(s_t)$, and Log-COP-TD($\lambda$,$\beta$) obtains a parametric form that depends on the original features instead of a different set.


\subsection{Approximation Hardness}
As we propose to use linear function approximation for $\rho_d(s)$ and $\log\left(\rho_d(s)\right)$ one cannot help but wonder how hard it is to approximate these quantities, especially compared to the value function. The comparison between $V(s)$ and $\rho_d(s)$ is problematic for several reasons: 
\begin{enumerate}
	\item The ultimate goal is estimating $V^{\pi}(s)$, approximation errors in $\rho_d(s)$ are second order terms.
	\item The value function $V^{\pi}(s)$ depends on the policy-induced reward function and transition probability matrix, while $\rho_d(s)$ depends on the stationary distributions induced by both policies. Since each depends on at least one distinct factor - we can expect different setups to result in varied approximation hardness. For example, if the reward function has a poor approximation then so will $V^{\pi}(s)$, while extremely different behavior and target policies can cause $\rho_d(s)$ to behave erratically.
	\item Subsequently, the choice of features for approximating $V^{\pi}(s)$ and $\rho_d(s)$ can differ significantly depending on the problem at hand. 
\end{enumerate}
If we would still like to compare $V^{\pi}(s)$ and $\rho_d(s)$, we could think of extreme examples:
\begin{itemize}
	\item When $\pi=\mu$, $\rho_d(s) \equiv 1$, when $R(s) \equiv 0$ then $V^{\pi}(s) \equiv 0$.
	\item In the chain MDP example in Section \ref{Sec:Motiv} we saw that $\rho_d(s)$ is an exponential function of the location in the chain. Setting reward in one end to $1$ will result in an exponential form for $V^{\pi}(s)$ as well. Subsequently, in the chain MDP example approximating $\log\left(\rho_d(s)\right)$ is easier than $\rho_d(s)$ as we obtain a linear function of the position; This is not the general case.
\end{itemize}

\section{Experiments} \label{Subsec:beta}
We have performed 3 types of experiments. Our first batch of experiments (Figure \ref{Fig:rho_d}) demonstrates the accuracy of predicting $\rho_d$ by both COP-TD($\lambda$, $\beta$) and Log-COP-TD($\lambda$, $\beta$). We show two types of setups in which visualization of $\rho_d$ is relatively clear - the chain MDP example mentioned in Section \ref{Sec:Motiv} and the mountain car domain \cite{sutton_reinforcement_1998} in which the state is determined by only two continuous variables - the car's position and speed. The parameters $\lambda$ and $\beta$ exhibited low sensitivity in these tasks so they were simply set to $0$, we show the estimated $\rho_d$ after $10^6$ iterations. For the chain MDP (top two plots, notice the logarithmic scale) we first approximate $\rho_d$ without any function approximation (top-left) and we can see COP-TD manages to converge to the correct value while Log-COP-TD is much less exact. When we use linear feature space (constant parameter and position) Log-COP-TD captures the true behavior of $\rho_d$ much better as expected. The two lower plots show the error (in color) in $\rho_d$ estimated for the mountain car with a pure exploration behavior policy vs. a target policy oriented at moving right. The z-axis is the same for both plots and it describes a much more accurate estimate of $\rho_d$ obtained through simulations. The features used were local state aggregation. We can see that both algorithms succeed similarly on the position-speed pairs which are sampled often due to the behavior policy and the mountain. When looking at more rarely observed states, the estimate becomes worse for both algorithms, though Log-COP-TD seems to be better performing on the spike at position $> 0$. 
\begin{figure} 
	\caption{Estimation quality of COP-TD and Log-COP-TD in the chain MDP (top) and mountain car (bottom) problems. The chain MDP plots differ by the function approximation and the shading reflects one standard deviation over 10 trajectories. The mountain car plots compare COP-TD with Log-COP-TD where the z-axis is the same (true $\rho_d$) with the colors specifying the error. }
	\label{Fig:rho_d}
	\centering  
	\hspace*{-0.65cm}\includegraphics[width=0.57\textwidth]{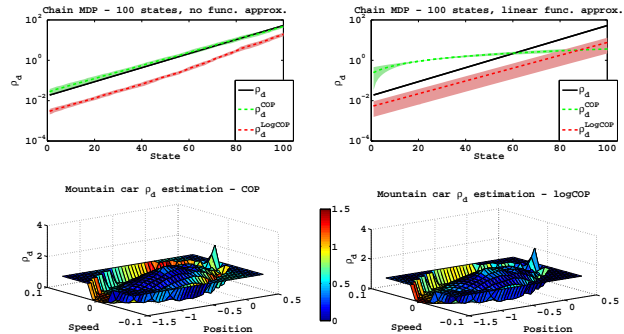}
\end{figure}

Next we test the sensitivity of COP-TD($\lambda$, $\beta$) and Log-COP-TD($\lambda$,$\beta$) to the parameters $\beta$ and $\gamma_{\log}$ (Figure \ref{Fig:lambda_beta}) on two distinct toy examples - the chain MDP introduced before but with only 30 states with the position-linear features, and a random MDP with 32 states, 2 actions and a $5$-bit binary feature vector along with a free parameter (this compact representation was suggested by \citet{white2016investigating} to approximate real world problems). The policies on the chain MDP were taken as described before, and on the random MDP a state independent $0.75$/$0.25$ probability to choose an action by the behavior/target policy. As we can see, larger values of $\beta$ cause noisier estimations in the random MDP for COP-TD($\lambda$, $\beta$), but has little effect in other venues. As for $\gamma_{\log}$ - we can see that if it is too large or too small the error behaves sub-optimally, as expected for the crude approximation of Equation \ref{Eq:logApprox}. In conclusion, unlike ETD($\lambda$, $\beta$), Log/COP-TD($\lambda$, $\beta$) are much less effected by $\beta$, though $\gamma_{\log}$ should be tuned to improve results. 

\begin{figure} 
	\caption{The effect of $\beta, \gamma_{\log}$ on COP-TD($\lambda$,$\beta$) and Log-COP-TD($\lambda$,$\beta$), the y-axis is $\rho_d$'s estimation sum of squared errors (SSE) over all states.}
	\label{Fig:lambda_beta}
	\centering  
	\hspace*{-0.4cm}\includegraphics[width=0.5\textwidth,height=7cm]{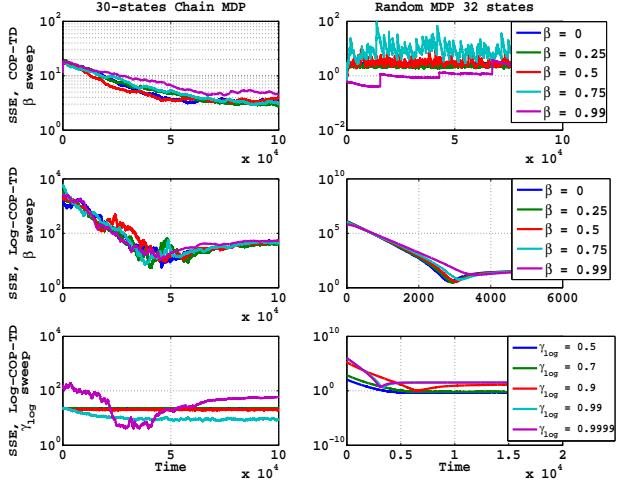}
\end{figure}

Our final experiment (Figure \ref{Fig:comp}) compares our algorithms to ETD($\lambda$, $\beta$) and GTD($\lambda$, $\beta$) over 4 setups: chain MDP with 100 states with right half rewards $1$ with linear features, a 2 action random MDP with 256 states and binary features, acrobot (3 actions) and cart-pole balancing (21 actions) \cite{sutton_reinforcement_1998} with reset at success and state aggregation to $100$ states. In all problems we used the same features for $\rho_d$ and $V^\pi(s)$ estimation, $\gamma=0.99$, constant step size $0.05$ for the TD process and results were averaged over 10 trajectories, other parameters ($\lambda$, $\beta$, other step sizes, $\gamma_{\log}$) were swiped over to find the best ones. To reduce figure clutter we have not included standard deviations though the noisy averages still reflect the variance in the process. Our method of comparison on the first 2 setups estimates the value function using the suggested algorithm, and finds the $d_\pi$ weighted average of the error between $V$ and the on-policy fixed point $\Pi_\pi T V_\pi$: 
\begin{equation} \label{Eq:err}
	\| \hat{V} -  \Pi_\pi T V_\pi \|^2_{d_\pi} = \sum_s d_\pi(s) \left[ ( \theta^* - \hat{\theta} )^\top \phi(s) \right]^2,
\end{equation}
where $\theta^*$ is the optimal $\theta$ obtained by on-policy TD using the target policy. On the latter continuous state problems we applied on-line TD on a different trajectory following the target policy, used the resulting $\theta$ value as ground truth and taken the sum of squared errors with respect to it. The behavior and target policies for the chain MDP and random MDP are as specified before. For the acrobot problem the behavior policy is uniform over the 3 actions and the target policy chooses between these with probabilities $(\frac{1}{6}, \frac{1}{3}, \frac{1}{2})$. For the cart-pole the action space is divided to 21 actions from -1 to 1 equally, the behavior policy chooses among these uniformly while the target policy is 1.5 times more prone to choosing a positive action than a negative one.

\begin{figure} 
	\caption{Error over time of several on-line off-policy algorithms.}
	\label{Fig:comp}
	\centering  
	 \hspace*{-0.65cm}\includegraphics[width=0.56\textwidth]{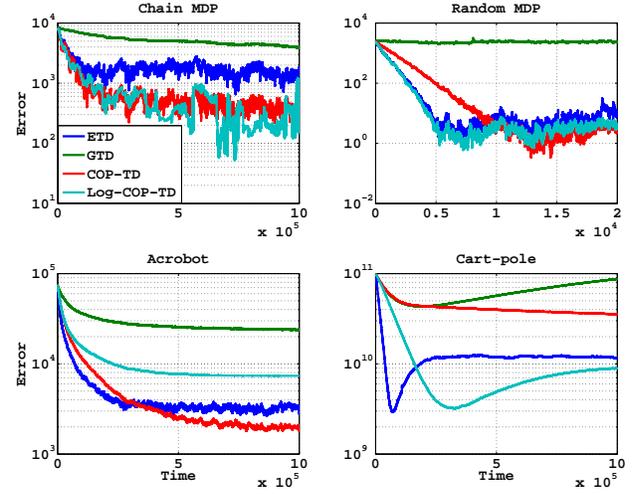}
\end{figure}

The experiments show that COP-TD($\lambda$, $\beta$) and Log-COP-TD($\lambda$, $\beta$) have comparable performance to ETD($\lambda$, $\beta$) where at least one is better in every setup. The advantage in the new algorithms is especially seen in the chain MDP corresponding to a large discrepancy between the stationary distribution of the behavior and target policy. GTD($\lambda$) is consistently worse on the tested setups, this might be due to the large difference between the chosen behavior and target policies which affects GTD($\lambda$) the most.

\section{Conclusion}
Research on off-policy evaluation has flourished in the last decade. While a plethora of algorithms were suggested so far, ETD($\lambda$, $\beta$) by \citet{hallak2015generalized} has perhaps the simplest formulation and theoretical properties. Unfortunately, ETD($\lambda$, $\beta$) does not converge to the same point achieved by on-line TD when linear function approximation is applied. 

We address this issue with COP-TD($\lambda$,$\beta$) and proved it can achieve consistency when used with a correct set of features, or at least allow trading-off some of the bias by adding or removing features. Despite requiring a new set of features and calibrating an additional update function, COP-TD($\lambda$,$\beta$)'s performance does not depend as much on $\beta$ as ETD($\lambda$,$\beta$), and shows promising empirical results. 

We offer a connection to the statistical interpretation of TD($\lambda$) that motivates our entire formulation. This interpretation leads to two additional approaches: (a) weight the $\Gamma^n_t$ using estimated variances instead of $\beta$ exponents and (b) approximating $\log[\rho_d]$ instead of $\rho_d$; both approaches deserve consideration when facing a real application.

\bibliography{ContractionBib}
\bibliographystyle{icml2017}

\newpage
\onecolumn
\section{Appendix}
\begin{table*}[!htbp]
	\caption{Notation table}
	
	\begin{center}	
		\begin{tabular}{ | l | l | }
			
			\hline
			$\lambda, \beta, \gamma_{\log}$ & Free parameters of TD algorithms mentioned in the paper \\ \hline	    
			$\mathcal{S}$ & State space \\ \hline
			$\mathcal{A}$ & Action space \\ \hline 
			$\mathcal{P}, P(s'|s,a)$ & Transition probability distribution \\	    \hline	   	   
			$\mathcal{R}, r(s,a)$ & Reward probability distribution \\ \hline 
			$\zeta$ & Distribution of the first state in the MDP \\	    \hline	    
			$\gamma$ & Discount factor \\ \hline
			$r_t = t(s_t,a_t)$ & Reward at time $t$, obtained at state $s_t$ and action $a_t$ \\ \hline	    
			$\mu(a|s)$ & Behavior policy (which generated the samples) \\ \hline	  	    
			$\pi(a|s)$ & Target policy \\ \hline	  	    	    
			$V^\pi(s)$ & Value function of state $s$ by policy $\pi$ \\ \hline	  	    	    
			$T$ & Bellman operator \\ \hline	  	    	    
			$R_\pi, P_\pi, T_\pi$  & Induced reward vector, transition matrix and Bellman operator by policy $\pi$ \\ \hline	  	    	    
			$R^{(n)}_{t, s_t}, R^\lambda_{t, s_t}$ & Value function estimates used in TD($\lambda$)\\ \hline	  	    	    
			$T^\lambda$ & Underlying TD($\lambda$) operator \\ \hline	  	    	    
			$d_\pi(s)$ & $\pi$-induced stationary distributions on the state space by policy  \\ \hline	   	    
			$\phi(s)$ & Feature vector of state $s$ \\ \hline 		
			$\theta$ & Weight vector for estimating $V(s)$ \\	    \hline	   	   
			$\rho_t$ & One-step importance sampling ratio \\	    \hline
			$\Gamma^n_t$ & $n$-steps importance sampling ratio \\   \hline	   	   		   	   	   	      	   
			$\rho_d$ & Stationary distribution ratio \\ \hline
			$\Phi$ & The feature matrix for each state \\	    \hline  	     	    
			$\tilde{\Gamma}^n_t$ & Estimated probabilities ratio \\	    \hline  	     	    
			$\tilde{\Gamma}^\beta_t$ & Weighted estimated probabilities ratio \\	    \hline  	   
			$\triangle_{d_\mu}$ & $d_\mu$ weighted simplex \\	    \hline  	     	    	
			$\alpha_t$ & Learning rate \\	    \hline	   	   	   	    	   	
			$Y, Y^\beta$ & COP operators, underlying COP-TD($\lambda$, $\beta$) \\	    \hline		   	
			$\theta_\rho$ & Weight vector for estimating $\rho_d$ \\ \hline

		\end{tabular}
	\end{center}
\end{table*}

Assumptions:
\begin{enumerate}
	\item Under both policies the induced Markov chain is ergodic.
	\item The first state $s_0$ is distributed according to the behavior policy $d_\mu(s)$.
	\item The support of $\mu$ contains the support of $\pi$, i.e.  $\pi(a|s) > 0 \Rightarrow \mu(a|s) > 0 $.
	\item The feature matrix $[\Phi]_{s,:} \doteq \phi(s)$ has full rank.
\end{enumerate}

\subsection{Proof of Lemma \ref{Lem:Main}}	
\textit{If the step sizes $\alpha_t$ hold $\sum_{t=0}^\infty \alpha_t = \infty, \sum_{t=0}^\infty \alpha^2_t < \infty$ then the process described by Equation \ref{Eq:Target} converges almost surely to the fixed point of $\Pi_\pi T_\pi V = V$. }
\begin{proof}
	Similarly to on-policy TD, we define $A$ and $b$, the fixed point is the solution to $A \theta = b$. First we find $A$ and show stability:
	\begin{equation}
	\begin{split}
	A & = \lim\limits_{t\rightarrow \infty} \mathbb{E}_\mu \left[ \rho_t\rho_d(s_t)\phi_t(\phi_t - \gamma \phi_{t+1})^\top \right] \\
	\quad & = \sum_s d_\mu (s) \rho_d(s) \mathbb{E}_\mu \left[  \rho_k \phi_k (\phi_k - \gamma \phi_{k+1})^\top | s_k = s \right] \\
	\quad & = \sum_s d_\pi(s)  \mathbb{E}_\pi \left[  \rho_k \phi_k (\phi_k - \gamma \phi_{k+1})^\top | s_k = s \right] \\
	\quad & = \Phi^\top D_\pi (I - \gamma P_\pi) \Phi.
	\end{split}
	\end{equation}
	This is exactly the same $A$ we would have obtained from TD($0$) and it is negative definite (see \cite{SuttonMW15}). Similarly we can find $b$:
	\begin{equation}
	\begin{split}
	b & = \lim\limits_{t\rightarrow \infty} \mathbb{E}_\mu \left[ \rho_t\rho_d(s_t)\phi_t r_t ^\top | s_k = s \right] \\
	\quad & = \sum_s d_\mu (s) \rho_d(s) \mathbb{E}_\mu \left[  \rho_k \phi_k r_k  | s_k = s \right] \\
	\quad & = \sum_s d_\pi(s)  \mathbb{E}_\pi \left[  \phi_k r_k | s_k = s \right] \\
	\quad & = \Phi^\top D_\pi R_\pi,
	\end{split}
	\end{equation}	
	
	and we obtained the same $b$ as on-policy TD($0$) with $\pi$. 
	
	Now we consider the noise of this off-policy TD, which is exactly the same noise as the on-policy TD only multiplied by $\rho_t \rho_d(s_t)$ - as long as the noise term of the ODE formulation \cite{kushner2003stochastic} is still bounded, the proof is exactly the same. According to Assumption 1, we know that $\rho_d$ is lower and upper bounded. By Assumption 3 we also know that $\rho_t$ is lower and upper bounded. Therefore the noise of the new process is bounded and the same a.s. convergence applies as on-policy TD($0$) \cite{tsitsiklis1997analysis}. Since $A,b$ are the same as on-policy TD($0$) for the target policy $\pi$, the convergence is to the same fixed point.
\end{proof}

\subsection{Proof of Lemma \ref{Lem:Gamma}}
\textit{Let $\widehat{\rho_d}$ be an unbiased estimate of $\rho _d$, and for every $n=0,1,\dots,t$ define  $\tilde{\Gamma}_t^n\doteq \widehat{\rho_d} (s_{t-n}) \Gamma^n_t$. Then:
	\begin{equation*}
	\mathbb{E}_\mu \left[ \tilde{\Gamma}^n_t | s_t \right] = 	\rho_d(s_t).
	\end{equation*}}
\begin{proof}
For any function on the state space $u(s)$:
\begin{equation}
\begin{split}
\mathbb{E}_\mu \left[ \Gamma^n_t u(s_{t-n}) | s_t \right] =&  \sum_{ \left( s_{i} \right)_{i=t-n}^{t-1} } \Pr_\mu ( \left( s_{i} \right)_{i=t-n}^{t-1} | s_t ) \Gamma^n_t u(s_{t-n}) \\
=& \sum_{ \left( s_{i} \right)_{i=t-n}^{t-1} } \frac{\Pr_\mu ( \left( s_{i} \right)_{i=t-n}^{t-1}, s_t )}{\Pr_\mu(s_t)} \Gamma^n_t u(s_{t-n}) \\
=& \sum_{ \left( s_{i} \right)_{i=t-n}^{t-1} } \frac{\Pr_\mu(s_{t-n}) \Pr_\pi( \left( s_{i} \right)_{i=t-n}^{t-1}, s_t |s_{t-n})} {\Pr_\mu(s_t)} u(s_{t-n}) \\
=& \sum_{ s_{t-n} } \frac{\Pr_\mu(s_{t-n}) \Pr_\pi( s_t |s_{t-n})} {\Pr_\mu(s_t)} u(s_{t-n}) \\	
=& u^\top D_\mu P^n_\pi D_\mu^{-1} e_{s_t},
\end{split}
\end{equation}
where $e_{s_t}$ is the unit vector of state $s_t$. So, for an unbiased estimate of $\rho _d$ denoted $\widehat{\rho_d}$ we can define and derive:
\begin{equation} \label{Eq:GammaGal} 
\begin{split}
\tilde{\Gamma}_t^n \doteq & \widehat{\rho_d} (s_{t-n}) \Gamma^n_t = \widehat{\rho_d} (s_{t-n}) \prod_{i=0}^{n-1} \rho_{t-i-1}, \\
\Rightarrow & \quad \mathbb{E}_\mu \left[ \tilde{\Gamma}^n_t | s_t \right] = \mathbb{E} \left[ \widehat{\rho_d} \right]^\top D_\mu P^n_\pi D_\mu^{-1} e_{s_t} \\
&= \rho_d^\top D_\mu P^n_\pi D_\mu^{-1} e_{s_t} \\
&= d_\pi^\top P^n_\pi D_\mu^{-1} e_{s_t} \\
&= d_\pi^\top D_\mu^{-1} e_{s_t} =	\rho_d(s_t).
\end{split}
\end{equation}
\end{proof}

\subsection{Proof of Lemma \ref{Lem:Contraction}}	
Under the ergodicity assumption, denote the eigenvalues of $P_\pi$ by $0 \leq \dots \leq | \xi_2 | < \xi_1 = 1$. Then $Y^\beta$ is a $\max_{i \neq 1} \frac{(1-\beta)|\xi_i|}{|1-\beta\xi_i|}$-contraction in the $L_2$-norm on the orthogonal subspace to $\rho_d$, and $\rho_d$ is a fixed point of $Y^\beta$.
\begin{proof}
	We first show that $\rho_d$ is  a fixed point of $Y^\beta$:
	\begin{equation}
	\begin{split}
	Y^\beta \rho_d &= (1-\beta) D^{-1}_\mu P^\top_\pi (I - \beta P^\top_\pi)^{-1} D_\mu \left( D^{-1}_\mu d_\pi \right) \\
	&= (1-\beta) D^{-1}_\mu P^\top_\pi (I - \beta P^\top_\pi)^{-1} d_\pi \\
	&= (1-\beta) D^{-1}_\mu (1-\beta)^{-1} d_\pi \\	
	&= D^{-1}_\mu d_\pi = \rho_d
	\end{split}
	\end{equation}
	Due to similarity, the eigenvalues of $Y^\beta$ are the same as these of $(1-\beta) P^\top_\pi (I - \beta P^\top_\pi)^{-1}$ which is a stochastic matrix with eigenvalues $\left( \frac{(1-\beta)\xi_i}{1-\beta\xi_i} \right)_{i=1}^{|S|}$, where for $i=1$ we obtain the eigenvalue $1$. Now for every vector orthogonal to $\rho_d$ denoted $u$, the first eigenvalue has no effect on its spectral decomposition, which means that $\| Y^\beta u \| \leq \max_{i \neq 1} \frac{(1-\beta)|\xi_i|}{|1-\beta\xi_i|} \|u \|$.
\end{proof}	

\subsection{Proof of Theorem \ref{Thm:COP1}}	
\textit{If the step sizes satisfy $\sum_t \alpha_t = \sum_t \alpha^d_t = \infty, \sum_t (\alpha^2_t + (\alpha_t^d)^2) < \infty, \frac{\alpha_t}{\alpha^d_t} \rightarrow 0, t \alpha^d_t \rightarrow 0$, and $\mathbb{E} \left[ (\beta^n \Gamma_t^n)^2 | s_t \right] \leq C$ for some constant $C$ and every $t,n$, then after applying COP-TD($0$, $\beta$), $\widehat{ \rho_d }_{,t}$  converges to $\rho_d$ almost surely, and $\theta_t$ converges to the fixed point of $\Pi_\pi T_\pi V$. }
\begin{proof}
	We use a three timescales stochastic approximation analysis. The fastest process is $\hat{d}_\mu (s)$ which converges naturally with time-step $O(\frac{1}{t})$:
	
	\begin{equation}
	\hat{d}_{\mu, t+1} = \frac{1}{t+1}\sum_{k=0}^{t}e_{s_k} = \frac{1}{t+1} (t \hat{d}_{\mu, t}  + e_{s_t}) = \hat{d}_{\mu, t} + \frac{1}{t+1} (  e_{s_t} - \hat{d}_{\mu, t}).
	\end{equation}
	The process $\hat{d}_{\mu, t}$ converges almost surely to $d_\mu$ by the strong law of large numbers. 
		Our next process is $\widehat{\rho_d}_{,t}$, which we will show converges a.s. to $\rho_d$ with $\hat{d}_\mu (s) = d_\mu(s)$:

	\begin{lemma}
		The process:
		\begin{equation}
		\begin{split}
		F_t &= \rho_{t-1} (\beta F_{t-1} + e_{s_{t-1}}), \quad\quad\quad n(\beta) = \beta n(\beta) + 1 \\
		\widehat{ \rho_d }_{,t+1}(s_t) &= \Pi_{\Delta_{d_\mu}} \left( \widehat{ \rho _d }_{,t}(s_t) + \alpha^d_t \left( \frac{F^\top_t \widehat{ \rho _d }_{,t}}{n(\beta)} - \widehat{ \rho _d }_{,t}(s_t) \right) \right)
		\end{split}
		\end{equation}
		Converges almost surely to $\rho_d$.
	\end{lemma}
	\begin{proof}
		We follow the notation from \cite{schuss2009stochastic}. We first specify the stochastic approximation using $h(x)$ and $M_{n+1}$:
		\begin{equation}
		\begin{split}
		h(\widehat{ \rho _d }_{,t}) =& \mathbb{E} \left[ \frac{F^\top_t \widehat{ \rho _d }_{,t}}{n(\beta)} |s_t \right] - \widehat{ \rho _d }_{,t}(s_t) = e_{s_t} Y^\beta \widehat{ \rho _d }_{,t} - e_{s_t} \widehat{ \rho _d }_{,t} = e_{s_t} ( Y^\beta - I) \widehat{ \rho _d }_{,t}, \\
		M_{n+1} =& \frac{F^\top_t \widehat{ \rho _d }_{,t}}{n(\beta)} - \mathbb{E} \left[ \frac{F^\top_t \widehat{ \rho _d }_{,t}}{n(\beta)} |s_t \right] = \frac{F^\top_t \widehat{ \rho _d }_{,t}}{n(\beta)} - e_{s_t} Y^\beta \widehat{ \rho _d }_{,t}.
		\end{split}
		\end{equation}
		Now there are several conditions that must follow - condition on the step sizes, conditions on the Martingale and conditions on the projection. If all of these are met than the process converegs to the fixed point of the projected operator $\rho_d$.
		
		The step size conditions follow by the theorem's assumption. Now we move on to the Martingale conditions. 
		
		Obviously $\mathbb{E} \left[M_{n+1} | s_t\right] = 0$. In order for $\mathbb{E} \left[ \| M_{n+1} \|^2 | s_t\right]$ to be bounded a.s., we use the assumption $\mathbb{E} \left[ (\beta^n \Gamma^n)^2 \right] \leq C$. Since $F_t$ is the leading factor in $\mathbb{E} \left[ \| M_{n+1} \|^2 | s_t\right]$ (the others are naturally bounded depending quadratically on $\widehat{\rho_d}$), and $F_t = \tilde{\Gamma}^\beta = (1-\beta) \sum_{n=0}^\infty \beta^n \tilde{\Gamma}^{n+1}_t$, the upper bound follows. 
		
		Now let's consider the projection where we follow the discussion in \cite{schuss2009stochastic}, Section 5.4. Notice that $h$ is Lipschitz and the eigenvalues around the fixed point are non-negative, therefore there's a stable invariant solution set. In addition, the projection is to a closed convex set $\Delta_{d_\mu}$, so $\widehat{\rho_{d,t}}$ is bounded. Hence, our goal is to show that the projection is Lipschitz and that we can ignore its non-smooth boundary.
		
		The projection to the simplex zeros some coordinates and decreases a constant from the other coordinates. If indeed it zeros some coordinates - we are at a problematic area of the space since the projection is not Frechet differentiable there (we are on the boundary of the set). However, because $\rho_d(s) > 0$ (Assumption 1), the unprojected ODE repels $\widehat{\rho_d}$ from these problematic boundaries, and we can assume that after enough the steps the projection is simply a projection to the affine subspace $\sum_s d_\mu (s) u(s) = 1$. In that case the projection is given by: $\Pi_{\Delta_{d_\mu}} u = (I - \frac{1}{\| d_\mu \|^2} d_\mu d_\mu^\top) (u - \textbf{1}) + \textbf{1}$ and its Frechet derivative is $\bar{\Pi}_{\Delta_{d_\mu}}u = (I - \frac{1}{\| d_\mu \|^2} d_\mu d_\mu^\top)u $. This derivative is Lipschitz continuous which means that its composition with $h$ is also Lipschitz . 
		
		Hence, the process converges to the solution set of $\bar{\Pi}_{\Delta_{d_\mu}}h(x)=0$ for $\Pi_{\Delta_{d_\mu}}x=x$. Under these constraints the only fixed point can be $\rho_d$ (intersection of the $c \cdot \rho_d$ line with the weighted simplex set $\Delta_{d_\mu})$.
		
	\end{proof}	
	Finally, treating the $\theta_t$ process assuming $\widehat{\rho_d}_{,t}$ already converged to $\rho_d$, leaves us with Lemma \ref{Lem:Main}. Since each process depends only on the previous ones, it is enough to show they converge independently as long as the step sizes satisfy the rate constraints. 
\end{proof}

\subsection{Proof of Theorem \ref{Thm:COP2}}	
\textit{If the step sizes hold $\sum_t \alpha_t = \sum_t \alpha^d_t = \infty, \sum_t (\alpha^2_t + (\alpha_t^d)^2) < \infty, \frac{\alpha_t}{\alpha^d_t} \rightarrow 0, t \alpha^d_t \rightarrow 0$, and $\mathbb{E} \left[ (\beta^n \Gamma_t^n)^2 | s_t \right] \leq C$ for some constant $C$ and every $t,n$, then after applying COP-TD($0$, $\beta$) with function approximation satisfying $\phi_\rho(s) \in \mathbb{R}_+^k$, $\widehat{ \rho_d }_{,t}$  converges to the fixed point of $\Pi_{\Delta_{\mathbb{E}_\mu [\phi_\rho] }} \Pi_{\phi_\rho} Y^\beta $ denoted by $\rho^\text{COP}_d$ almost surely, and if $\theta_t$ converges it is to the fixed point of $\Pi_{d_\mu \circ \rho^\text{COP}_d} T_\pi V$. }
\begin{proof}
	Similarly to the proof of Theorem , we can analyze the system on 3-time scales. The fastest process is $\hat{d}_{\phi_\rho}$ which converges naturally with time-step $O(\frac{1}{t})$: 
	
	\begin{equation}
	\hat{d}_{\phi_\rho, t+1} = \frac{1}{t+1}\sum_{k=0}^{t}\phi_\rho(s_k) = \frac{1}{t+1} (t \hat{d}_{\phi_\rho, t}  + \phi_\rho(s_k)) = \hat{d}_{\phi_\rho, t} + \frac{1}{t+1} (  \phi_\rho(s_k) - \hat{d}_{\phi_\rho, t}).
	\end{equation}
	The process $\hat{d}_{\phi_\rho}$ converges almost surely to $\mathbb{E}_\mu [\phi_\rho(s)]$ by the strong law of large numbers. 
	
	We now show that $\widehat{ \rho_d }_{,t}$ converges to $\rho^\text{COP}_d$. The proof follows the same lines as that of Theorem \ref{Thm:COP1}, where two things changed: (a) The estimated value $\widehat{ \rho_d }_{,t}$ is now contained in a linear subspace spanned by $\Phi_\rho$, and (b) the projection changed as well from the $d_\mu$ simplex to the $\mathbb{E}_\mu [\phi_\rho(s)]$ simplex. 
	
	First we represent the corresponding $A$ and $b$ of the projected  $\rho_d$ ODE as follows (we assume $\beta=0$, but the results are similar for general $\beta$):
	\begin{equation}
	A = \Phi_\rho^\top D_\mu (D^{-1}_\mu P^\top_\pi D_\mu  - I) \Phi_\rho = \Phi_\rho^\top (P^\top_\pi  - I) D_\mu \Phi_\rho, \quad\quad\quad b = 0.
	\end{equation}
	We can now verify that a solution to $Ax=b$ also holds the Projected COP equation: $\Pi^{\phi_\rho}_{d_\mu} Y \Phi_\rho \theta_\rho = \Phi \theta_\rho$ by multiplying it from the left by $\Phi_\rho^\top D_\mu$:
	\begin{equation}
	\begin{split}
	\Phi_\rho^\top D_\mu \left[ \Pi^{\phi_\rho}_{d_\mu} Y \Phi_\rho \theta_\rho - \Phi_\rho \theta_\rho \right] &= \Phi_\rho^\top D_\mu \left[ \Phi_\rho  \left( \Phi^\top_\rho D_\mu \Phi_\rho \right)^{-1} \Phi^\top_\rho   D_\mu \left(D^{-1}_\mu P_\pi D_\mu \right) \Phi_\rho \theta_\rho - \Phi_\rho \theta_\rho \right] \\
	&= \Phi^\top_\rho P_\pi D_\mu \Phi_\rho \theta_\rho - \Phi^\top_\rho D_\mu \Phi_\rho \theta_\rho \\
	&= \Phi_\rho^\top (P_\pi  - I) D_\mu \Phi_\rho
	\end{split}
	\end{equation}
	Let's look on the new projection $\Pi_{\Delta_{\mathbb{E}_\mu [\phi_\rho] }}$. Since we demanded $\phi_\rho(s) \in \mathbb{R}_+^k$, this set is close and bounded, so the convergence is guaranteed. 
	
	In order for the new projection $\Pi_{\Delta_{\mathbb{E}_\mu [\phi_\rho] }}$ to be Frechet differentiable, we should verify it still avoids the boundaries meaning $\theta_\rho > 0$ coordinate-wise. However, even were this not true, we could simply throw away one of the features and get a smaller problem that does hold this condition, keeping the projection Frechet differentiable similarly to before. Subsequently $\widehat{ \rho_d }_{,t}$  converges to the fixed point of $\Pi_{\Delta_{\mathbb{E}_\mu [\phi_\rho] }} \Pi_{\phi_\rho} Y^\beta $ denoted by $\rho^\text{COP}_d$ almost surely.
	
	Moving on to the last process, we get the following equations:
	\begin{equation}
		\begin{split}
			A & = \lim\limits_{t\rightarrow \infty} \mathbb{E}_\mu \left[ \rho_t\rho^\text{COP}_d(s_t)\phi_t(\phi_t - \gamma \phi_{t+1})^\top \right]  = \Phi^\top D_\mu \text{diag}(\rho_d^\text{COP}) (I - \gamma P_\pi) \Phi, \\
			b &= \lim\limits_{t\rightarrow \infty} \mathbb{E}_\mu \left[ \rho_t\rho_d^\text{COP}(s_t)\phi_t r_t ^\top | s_k = s \right] = \Phi^\top D_\pi \text{diag}(\rho_d^\text{COP}) R_\pi,			
		\end{split}
	\end{equation}
	leading us to the known convergence solution (if indeed the process converge, which is not necessarily true) which is the fixed point of $\Pi_{d_\mu \circ \rho^\text{COP}_d} T_\pi V$.
	
\end{proof}
\subsection{Proof of Corollary 1}	
\textit{Let $0<\epsilon<1$. If $ (1-\epsilon) \rho_d \leq \rho_d^\text{COP} \leq (1+\epsilon) \rho_d$, then the fixed point of COP-TD($0$,$\beta$)  with function approximation $\theta^\text{COP}$ satisfies the following, where $\| \cdot \|_\infty$ is the $L_\infty$ induced norm:
	\begin{equation}
	\| \theta^* - \theta^\text{COP} \|_\infty \leq \epsilon \| A_\pi^{-1} \Phi^\top \| _\infty \left(   R_\text{max}  +  (1+\gamma) \| \Phi \|_\infty \| \theta^\text{COP} \|_\infty \right),
	\end{equation}
	where $A_\pi = \Phi^\top D_\pi (I - \gamma P_\pi) \Phi$, and $\theta^*$ sets the fixed point of the operator $\Pi_{d_\pi} T_\pi V$.}
\begin{proof}
	If $ (1-\epsilon) \rho_d \leq \rho_d^\text{COP} \leq (1+\epsilon) \rho_d$, then we know that the weights of the projection hold:
	\begin{equation}
	(1-\epsilon)d_\pi \leq d_{\tilde{\pi}} \doteq d_\mu \circ \rho^\text{COP}_d  \leq (1-\epsilon)d_\pi.
	\end{equation}
	
	We write the solution equations for both weight vectors
	\begin{equation}
	\begin{split}
	(\Phi^\top D_\pi (I - \gamma P_\pi) \Phi) \theta^* = \Phi^\top D_\pi R \\
	(\Phi^\top D_{\tilde{\pi}} (I - \gamma P_\pi) \Phi) \theta^\text{COP} = \Phi^\top D_{\tilde{\pi}} R
	\end{split}
	\end{equation}
	
	Now we subtract both equations and add and subtract $(\Phi^\top D_\pi (I - \gamma P_\pi) \Phi) \theta^\text{COP}$:
	\begin{equation}
	(\Phi^\top D_\pi (I - \gamma P_\pi) \Phi) (\theta^* - \theta^{\text{COP}}) = \Phi^\top (D_\pi - D_{\tilde{\pi}}) R + (\Phi^\top (D_{\tilde{\pi}} - D_\pi) (I - \gamma P_\pi) \Phi) \theta^\text{COP}
	\end{equation}
	Now we take $L_\infty$ norm on both sides, and use induced matrix sub-multiplicative property:
	\begin{equation}
	\begin{split}
	\| \theta^* - \theta^\text{COP} \|_\infty =& \| (\Phi^\top D_\pi (I - \gamma P_\pi) \Phi)^{-1} \left(\Phi^\top (D_\pi - D_{\tilde{\pi}}) R + (\Phi^\top (D_{\tilde{\pi}} - D_\pi) (I - \gamma P_\pi) \Phi) \theta^\text{COP} \right) \|_\infty \\
	\leq &\|  (\Phi^\top D_\pi (I - \gamma P_\pi) \Phi)^{-1}  \|\Phi^\top\|_\infty \left( (D_\pi - D_{\tilde{\pi}}) R\|_\infty + \| ( (D_{\tilde{\pi}} - D_\pi) (I - \gamma P_\pi) \Phi) \theta^\text{COP} \|_\infty \right) \\
	\leq& \| A_\pi^{-1}\Phi^\top  \|_\infty \left( \| D_\pi - D_{\tilde{\pi}} \|_\infty \| R\|_\infty +  \| D_\pi - D_{\tilde{\pi}} \|_\infty \|I - \gamma P_\pi \|_\infty \| \Phi \|_\infty \| \theta^\text{COP} \|_\infty \right) \\
	\leq& \epsilon \| A_\pi^{-1}\Phi^\top \| _\infty \left(  R_\text{max}  +  (1+\gamma) \| \Phi \|_\infty \| \theta^\text{COP} \|_\infty \right)
	\end{split}
	\end{equation}

\end{proof}

\subsection{More details on the experiments}
Experiments for Figure \ref{Fig:rho_d}: 100 states chain MDP, with probability $0.51$ to move left / right for the behavior / target policy. Results were taken after $T=1e6$ iterations. For COP-TD we used $\beta=0$ and a constant step size $0.5$. For Log-COP-TD we used $\beta=0, \gamma_{\log} = 0.9999$ and constant step size $0.5$. The experiment was conducted 10 times and the standard deviation is given as shading in the graph. 

For the mountain car experiment we used the simulator given by \url{https://jamh-web.appspot.com/download.htm}. The state aggregation was obtained by running kmeans with $100$ clusters and taking the centers as representative states. The behavior policy was taken to be uniform over the 3 possible actions (-1, 0, 1), and the target policy chose these actions with probabilities (1/6,1/3,1/2) regardless of the state. COP-TD was applied with $\beta=0$ and constant step size $0.01$, and Log-COP-TD was applied with $\beta=0, \gamma_{\log} = 0.9$ and constant step size $0.01$. Both algorithms ran for $T=1e6$ iterations before the estimated $\rho_d$ was taken. 

Experiments for Figure \ref{Fig:lambda_beta}: 100 states chain MDP, with probability $0.51$ to move left / right for the behavior / target policy. For both algorithms constant step size $0.5$. For Log-COP-TD we used $\beta=0, \gamma_{\log} = 0.9999$ when sweeping on the other parameter. All experiments were conducted 10 times and we show the average result. 

In the randomized MDP with 32 states we used uniform distribution over transition probabilities for two possible actions, and the target policy had $p=0.75$ to choose one action where the behavior policy had $p=0.75$ to choose the other action. All experiments were conducted 10 times and we show the average result. 

Experiments for Figure \ref{Fig:comp} were obtained by running COP-TD, Log-COP-TD, ETD and GTD over 4 setups. The distinct parameters of each algorithm were swiped over to find the best value: COP-TD's step size and $\beta$, Log-COP-TD's steps size, $\beta$ and $\gamma_{\log}$, ETD's $\beta$ and GTD's step size. The step size of the main process and $\lambda$ were taken to be the same for all algorithms: step size = 0.05 and $\lambda=0$ (obtained also by sweeping over possible values). The simulators for acrobot and pole-balancing were taken from \url{https://jamh-web.appspot.com/download.htm}.

\end{document}